\documentclass[12pt]{article}

\usepackage[english]{babel}
\usepackage[utf8x]{inputenc}
\usepackage[T1]{fontenc}

\usepackage[letterpaper,top=1in,bottom=1in,left=1in,right=1in,marginparwidth=1.75cm]{geometry}

\usepackage{amsmath,amsfonts,amsthm,amssymb,mathrsfs,dsfont} 
\usepackage{mathtools}
\usepackage{graphicx}
\usepackage[colorinlistoftodos]{todonotes}
\usepackage[colorlinks=true, allcolors=blue]{hyperref}
\usepackage[capitalize,nameinlink]{cleveref}
\usepackage{verbatim}

\newcommand{\bE}{\mathbb{E}}
\newcommand{\R}{\mathbb{R}}
\newcommand{\F}{\mathcal{F}}
\newcommand{\eps}{\varepsilon}

\newcommand{\E}{\bE}      

\newcommand{\KL}{\mathop{\bf KL\/}}
\newcommand{\bone}{\boldsymbol{1}}

\newcommand{\vnote}[1]{\textcolor{red}{\small {\textbf{(Vishesh: }#1\textbf{) }}}}
\newcommand{\fnote}[1]{\textcolor{blue}{\small {\textbf{(Fred: }#1\textbf{) }}}}

\newtheorem{theorem}{Theorem}[section]
\newtheorem*{namedtheorem}{\theoremname}
\newcommand{\theoremname}{testing}

\newtheorem{lemma}[theorem]{Lemma}

\newtheorem*{question*}{Question}

\theoremstyle{definition}

\newtheorem{defn}[theorem]{Definition}
\newtheorem{remark}[theorem]{Remark}
\newtheorem{example}[theorem]{Example}

\theoremstyle{plain}

\usepackage{algpseudocode,algorithm,algorithmicx}

\title{The Mean-Field Approximation: Information Inequalities, Algorithms, and Complexity}
\author{Vishesh Jain\thanks{Massachusetts Institute of Technology. Department of Mathematics. Email: {\tt visheshj@mit.edu}} \and Frederic Koehler\thanks{Massachusetts Institute of Technology. Department of Mathematics. Email: {\tt fkoehler@mit.edu}} \and Elchanan Mossel\thanks{Massachusetts Institute of Technology. Department of Mathematics and IDSS. Supported by ONR grant N00014-16-1-2227   and 
NSF CCF-1665252 and DMS-1737944. Email: {\tt elmos@mit.edu} } }

\begin{document}
\maketitle
\thispagestyle{empty}
\setcounter{page}{0}

\begin{abstract}
The mean field approximation to the Ising model is a canonical variational tool that is used for analysis and inference in Ising models. We provide a simple and optimal bound for the KL error of the mean field approximation for Ising models on general graphs, and extend it to higher order Markov random fields. Our bound improves on previous bounds obtained in work in the graph limit literature by Borgs, Chayes, Lov{\'a}sz, S{\'o}s, and Vesztergombi and another recent work by Basak and Mukherjee. Our bound is tight up to lower order terms. 

Building on the methods used to prove the bound, along with techniques from combinatorics and optimization, 
we study the algorithmic problem of estimating the (variational) free energy for Ising models and general Markov random fields. For a graph $G$ on $n$ vertices and interaction matrix $J$ with Frobenius norm $\| J \|_F$, we provide algorithms that approximate the free energy within an additive error of $\epsilon n \|J\|_F$ in time $\exp(poly(1/\epsilon))$. We also show that approximation within $(n \|J\|_F)^{1-\delta}$ is NP-hard for every $\delta > 0$. Finally, we provide more efficient approximation algorithms,
which find the optimal mean field approximation, for ferromagnetic Ising models and for Ising models satisfying Dobrushin's condition.  
\end{abstract}

\newpage

\section{Introduction}
One of the most widely studied models in statistical physics is the Ising model. An \emph{Ising model} is specified by
a probability distribution on the discrete cube $\{\pm1\}^n$ of the form
\[ P[X = x] := \frac{1}{Z} \exp(\sum_{i,j} J_{i,j} x_i x_j) = \frac{1}{Z} \exp(x^T J x), \]
where the collection $\{J_{i,j}\}_{i,j\in\{1,\dots,n\}}$ are the entries of
an arbitrary real, symmetric matrix with zeros on the diagonal. The distribution $P$ is referred to as the \emph{Boltzmann distribution}. The normalizing constant $Z=\sum_{x\in\{\pm1\}^{n}}\exp(\sum_{i,j=1}^{n}J_{i,j}x_{i}x_{j})$
is called the \emph{partition function }of the Ising model and the quantity $\F := \log{Z}$ is known as the \emph{free energy}. 

The free energy is a key physical quantity. It provides important information 
about the structure of the Boltzmann distribution. Given a naturally growing family of (possibly weighted) graphs with adjacency matrices $M_n$, 
one of the main problems of interest in statistical physics is to compute the asymptotics of the (suitably renormalized) free energy of the sequence of Ising models $J_n(\beta) = \beta M_n$ in the $n \to \infty$ limit for all values of $\beta$, where $\beta > 0$ is a parameter referred to as the \emph{inverse temperature}. This is because understanding the behavior of the free energy reveals a wealth of information about the underlying Ising model. For instance, points of non-smoothness in the limiting free energy (as a function of $\beta$) reveal the location of \emph{phase transitions}, which typically correspond to significant
changes in the behavior of the underlying Boltzmann distribution e.g. the emergence of long-range correlations. In addition, many other quantities of interest (such as net magnetization) can be computed in terms of free energies.

Although originally introduced in statistical physics, Ising models and their generalizations have also found a wide range of applications in many different areas like statistics, computer science, combinatorics, social networks, and biology (see, e.g., the references in \cite{basak2017universality}). Studying the free energy is of great interest in many of these applications as well. For instance, consider the problem in combinatorial optimization of maximizing the quadratic form $x \mapsto x^T M x$ over the hypercube $\{\pm 1\}^{n}$; this is essentially the problem of estimating the cut norm of a matrix and has max-cut as the special case when all of the entries are negative. The free energy of the model with interaction matrix $J_{\beta}:=\beta M$ provides a natural tempering of this optimization problem in the following sense:   
\[ \frac{1}{\beta} \mathcal{F}_{\beta} = \frac{1}{\beta} \log \sum_{x\in \{\pm 1\}^{n}}\exp\left(\beta \sum_{i,j=1}^{n}M_{ij}x_{i}x_{j}\right) \to \max_{x \in \{\pm 1\}^n} \sum_{i,j = 1}^n M_{ij} x_i x_j \]
as $\beta \to \infty$. 

In fact for every finite $\beta$, the free energy corresponds
to the objective value of a natural optimization problem of its own.
More precisely, the free energy is characterized by the following \emph{variational principle} (dating back to Gibbs, see the references in \cite{ellis2007entropy}):
\begin{equation}
\label{eqn:free-energy-variational-char}
\F = \max_{\mu} \left[\sum_{i,j} J_{ij} \E_{\mu}[X_i X_j] + H(\mu)\right],
\end{equation}
where $\mu$ ranges over all probability distributions on the boolean hypercube $\{\pm 1\}^{n}$. This can be seen by noting that 
\begin{equation}
\label{eqn:free-energy-KL}
\KL(\mu ||P)=\F - \sum_{i,j} J_{ij} \E_{\mu}[X_i X_j] - H(\mu),
\end{equation}
and recalling that $\KL(\mu ||P) \geq 0$ with equality if and only if $\mu = P$.
 
By substituting $J = \beta M$ in equation \cref{eqn:free-energy-variational-char}, we see 
that the Boltzmann distribution is simply the maximum entropy distribution $\mu$ for a fixed value
of the expected energy $\E_{\mu}[x^T M x]$. Thus, studying the free
energy for different values of $\beta$ provides much richer information about
the optimization landscape of $x \mapsto x^T M x$ over the hypercube than just the maximum value, e.g., in the max-cut case, the free energies encode
information about non-maximal cuts as well (see e.g. \cite{borgs2012convergent} for related discussion).
Apart from the applications mentioned above, it is clear by definition that knowledge of the free energy (or equivalently, the partition function) allows one to perform fundamental inference tasks like computing marginals and posteriors in Ising models and their generalizations. Unfortunately, the partition function, which is defined as a sum of exponentially many terms, turns out to be both theoretically and computationally intractable. Closed form expressions for the partition function are extremely hard to come by; in fact, providing such an expression even for the Ising model on the standard $3$-dimensional lattice remains one of the most outstanding problems in statistical physics. From a computational perspective, it is known that exactly computing the partition function of an Ising model with $J$ the adjacency matrix of a nonplanar graph is NP-hard (\cite{Istrail2000StatisticalMT}), and that approximate sampling/approximating the partition function is still NP-hard, even e.g. in the case of graphs with degree bounded by a small constant (see \cite{sly-sun}).

\subsection{The mean-field approximation: structural results}
Since exact computations, either analytic or otherwise, are typically infeasible, it is natural to look at schemes for approximating the partition function or the free energy. The \emph{naive mean-field approximation} provides one of the simplest and most common methods for doing this.  

The mean-field approximation to the free energy (also referred to as the \emph{variational free energy}) is obtained by restricting the distributions $\mu$ in the variational characterization of the free energy (\cref{eqn:free-energy-variational-char}) to be product distributions. Accordingly, we define the \emph{variational free energy} by 
\[ \F^* := \max_{x \in [-1,1]^n} \left[\sum_{i,j} J_{ij}
      x_i x_j + \sum_i H\left(\frac{x_i +
        1}{2}\right)\right]. \] 
Indeed, if $\bar{x} = (\bar{x}_1,\dots,\bar{x}_n)$ is the optimizer in the above definition, then the product distribution $\nu$ on the boolean hypercube, with the $i^{th}$ coordinate having expected value $\bar{x}_i$, minimizes $\KL(\mu||P)$ among all product distributions $\mu$. Moreover, it is immediately seen from \cref{eqn:free-energy-KL} that the value of this minimum KL is exactly $\F - \F^*$. Thus, the quantity $\F - \F^*$, which measures the quality of the mean-field approximation, may be interpreted information theoretically as the divergence between the closest product distribution to the Boltzmann distribution and the Boltzmann distribution itself.  
 
Owing to its simplicity, the mean field approximation 
has long been used in statistical physics (see, e.g., \cite{parisi1988statistical} for a textbook treatment) and also in Bayesian statistics \cite{peterson-anderson,jordan1999introduction,wainwright-jordan-variational}, where it is one of the prototypical examples of a \emph{variational method}. As a variational method, the mean field approximation has the attractive property that it always gives a valid lower bound for the free energy. It is well known \cite{ellis-newman} that the mean field approximation
is very accurate for the Curie-Weiss model, which is the Ising
model on the complete graph (see also \cref{example:curie-weiss}
for a complete description of the model). On the other hand, it is also known (see e.g. \cite{DemboMontanari:10}) that for very sparse graphs like trees of bounded arity, this is not the case. 
In recent years, considerable effort has gone into bounding the error of the mean-field approximation on more general
graphs; we will give a detailed comparison of our results
with recent work in \cref{sec:previous-results}. Our main structural result is the following inequality, which gives an explicit bound on the error of the mean field approximation for general graphs:

\begin{theorem}\label{thm-main-structural-result} 
Fix an Ising model $J$ on $n$ vertices.
Let $\nu := \arg\min_{\nu} \KL(\nu || P)$, where $P$ is the Boltzmann distribution and the minimum ranges
over all product distributions. 
Then,
$$ \KL(\nu || P)  = \F - \F^{*} \leq 200 n^{2/3} \|J\|_F^{2/3} \log^{1/3}(n \|J\|_F + e).$$
\end{theorem}
Here, $\|J\|_F := \sqrt{\sum_{i,j}J_{i,j}^2}$ is the \emph{Frobenius norm} of the matrix $J$. 

This result is \emph{tight up to logarithmic factors}, not just for product distributions, but also for a large class of variational methods. In particular, this class includes approximation by bounded mixtures of product distributions (as considered in \cite{jaakkola1998improving}),
as well as (mixtures of) restricted classes of Ising models, e.g. Ising models on acyclic graphs (see the discussion of tractable families in \cite{wainwright-jordan-variational}). Some other
methods for estimating the free energy (such as the Bethe approximation and the method of \cite{risteski-ising}) optimize over \emph{pseudo-distributions}
of some form and so the theorem itself cannot be directly applied, but essentially the same obstruction should still apply.
\begin{theorem}\label{thm:variational-lb}
Let $(\mathcal{Q}_n)_{n = 0}^{\infty}$ be a sequence of families
of probability distributions on $\{\pm 1\}^n$ which are closed under the following two operations:
\begin{enumerate}
\item Conditioning on variables: if
  $Q \in \mathcal{Q}_n$, $i \in [n]$, and $x_i \in \{\pm 1\}$, then the conditional distribution of $X_{\sim i}$ under $Q$ given $X_i = x_i$, which is a
  probability distribution on $\{\pm 1\}^{n - 1}$, is in $\mathcal{Q}_{n - 1}$.
\item Taking products: if $Q_1 \in \mathcal{Q}_m$ and $Q_2 \in \mathcal{Q}_n$,
  then $Q_1 \times Q_2 \in \mathcal{Q}_m \times \mathcal{Q}_n$.
\end{enumerate}
Furthermore, suppose that $(\mathcal{Q}_n)_{i = 1}^\infty$ does not contain the class of all probability
distributions induced by Ising models. Then, there exists a sequence $(J_i)_{i = 1}^{\infty}$ of Ising models of increasing size $n_i$ and with Boltzmann distributions $P_{J_i}$ 
such that 
\[ \KL(Q_{n_i} || P_{J_i}) = \Omega(n_i^{2/3}\|J_{n_i}\|_F^{2/3}), \]
where $Q_{n_i} := \arg\min_{Q \in \mathcal{Q}_{n_i}} \KL(Q,P_{J_i})$.
\end{theorem}

Our methods extend in a straightforward manner not just to Ising models with external fields, but indeed to general higher order Markov random fields, as long
as we assume a bound $r$ on the order of the highest interaction (i.e. size of the largest hyper-edge). The
results also generalize naturally to the case of non-binary alphabets but for simplicity, we only discuss the binary case.
\begin{defn}
Let $J$ be an arbitrary function on the hypercube $\{ \pm 1\}^n$
and suppose that the degree of $J$ is $r$ i.e. the Fourier decomposition of $J$ is 
$J(x) = \sum_{\alpha \subset [n]} J_{\alpha} x^{\alpha}$
with $r = \max_{J_{\alpha} \ne 0} |\alpha|$.
The corresponding \emph{order $r$ (binary) Markov random field} is the probability distribution
on $\{\pm 1\}^n$ given by
\[ P(X = x) = \frac{1}{Z}\exp(J(x)) \]
where the normalizing constant $Z$ is referred to as the \emph{partition function}.
For any polynomial $J$ we define $J_{=d}$ to be its $d$-homogeneous part and
 $\|J\|_F$ to be the square root of the total Fourier energy of $J$ i.e. $\|J\|_F^2 := \sum_{\alpha} |J_{\alpha}|^2$.
\end{defn}
\begin{theorem}\label{thm-mrf-main-structural-result} 
Fix an order $r$ Markov random field $J$ on $n$ vertices.
Let $\nu := \arg\min_{\nu} \KL(\nu || P)$, where $P$ is the Boltzmann distribution and the minimum ranges
over all product distributions.
Then, 
$$ \KL(\nu || P)  = \F - \F^{*} \leq 2000r \max_{1 \le d \le r} d^{1/3}n^{d/3} \|J_{=d}\|_F^{2/3} \log^{1/3}(d^{1/3}n^{d/3} \|J_{=d}\|_F^{2/3} + e).$$
\end{theorem}

\subsection{Examples}
We give a few examples of natural families of Ising models in order to illustrate the consequences
of our bounds.
\begin{example}[Curie-Weiss]\label{example:curie-weiss}
As our first example, we show how our bounds imply classical
results about the Curie-Weiss model (see \cite{ellis-newman}), in which $J_{ij} = (\beta/2n)$
for $i \ne j$ and there is a uniform external field $h$. In this case, we can explicitly solve the variational problem;
indeed, by checking the first-order optimality condition (\cref{eqn:mean-field-equations}), we see that an optimal
product distribution with marginals $\E[X_i] = x_i$ must have 
$x_i = \tanh(\sum_{j: j \ne i} \beta x_j/n + h)$.
Furthermore, since $x_i < x_j$ implies that $\tanh(\sum_{k: k \ne i} \beta x_k/n + h) > \tanh(\sum_{k: k \ne j} \beta x_k/n + h)$, it follows that  we cannot have $x_i < x_j$ for any pair $(i,j)$. 
Therefore, the optimal product distribution has all marginals equal to $x$, where $x$ is a solution of
\[ x = \tanh((1 - 1/n) \beta x + h). \]
Taking $n \to \infty$ and $h = 0$, this correctly predicts a phase transition at $\beta = 1$; the mean field equations go from having just one solution ($x = 0$) to two additional ``symmetry-breaking'' solutions with $x \ne 0$. By \cref{thm-mrf-main-structural-result}, we see that for any 
constant $\beta,h$, the normalized free energy $\mathcal{F}/n$ agrees with $\mathcal{F}^*/n$ in the $n \to \infty$ limit with error decaying at least as fast as $\tilde{O}(n^{-1/3})$.
\end{example}
\begin{example}[Uniform edge weights on graphs of increasing degree]
Fix $\beta \in \mathbb{R}$ and a sequence of graphs
$(G_{n_i})_{i = 1}^{\infty}$ with the number of vertices $n_i$ going to infinity, and let $m_i$ be the corresponding number of edges. Then, it is natural to look at the model
with uniform edge weights equal to $\beta n_i/m_i$, since this makes
the maximum value of $x^T J x$ on the order of $\Theta(n_i)$, which is 
the same scale as the entropy term in the variational definition of the free energy (\cref{eqn:free-energy-variational-char}). We say the model is \emph{ferromagnetic} if $\beta > 0$ and \emph{anti-ferromagnetic} if $\beta < 0$. Observe that $\|J\|_F = \beta n_i/\sqrt{m_i}$, so that by \cref{thm-main-structural-result}, we have $|\mathcal{F}/n_i - \mathcal{F}^*/n_i| = O(n_i^{1/3}\log^{1/3}{n_i}/m_i^{1/3})$. In particular, this goes to $0$ as long as $m_i = \omega(n_i \log n_i)$.
\end{example}
\begin{example}[Uniform edge weights on $r$-uniform hypergraphs]
Fix $\beta \in \mathbb{R}$ and
let $(G_{n_i})_{i=1}^{\infty}$ be a sequence of $r$-uniform hypergraphs
with $n_i$ vertices and $m_i$ hyperedges. Analogous to the graph case, we let $J(x) = \frac{\beta n_i}{m_i} \sum_{S \in E(G_{n_i})} x_S$,
so that the maximum of $J$ is on the same order as the entropy term in the free energy. We still have $\|J\|_F = \beta n_i/\sqrt{m_i}$, and see by \cref{thm-mrf-main-structural-result} that $|\mathcal{F}/n_i - \mathcal{F}^*/n_i| = O(n_i^{(r - 1)/3}\log{n_i}/m_i^{1/3})$. This converges
to $0$ as long as $m_i = \omega(n_i^{r - 1} \log{n_i})$.
\end{example}
\subsection{Algorithmic results}
Next, we study the algorithmic aspects of the mean field approximation and variational methods. We begin by showing that in a certain \emph{high-temperature regime} (specifically, the range of parameters satisfying the \emph{Dobrushin uniqueness criterion} \cite{dobrushin-uniqueness}), the minimization problem
defining the variational free energy is convex.  
\begin{theorem}\label{thm:high-temperature-convex}
Suppose $J$ is the interaction matrix of an Ising model with
arbitrary external field $h_i$ at vertex $i$, and suppose that
for every row $i$ of $J$, we have $\sum_{j} 2|J_{ij}| \le 1$.
Then, the maximization problem defining the variational free energy is concave, and hence can be solved to additive $\epsilon$-error in time $poly(n,\log(1/\epsilon))$.
\end{theorem}

\begin{remark}
Note that in the literature (e.g. \cite{dobrushin-uniqueness}), the Dobrushin uniqueness criterion is stated as $\sum_j |J_{i,j}| \leq 1$. This corresponds to the above condition $\sum_j 2|J_{i,j}| \leq 1$ in our normalization, since we do not insert a factor of ${1/2}$ in front of the quadratic term in the definition of (variational) free energy.
\end{remark}
A well known heuristic for finding the optimal mean-field approximation (see, e.g., the discussion in \cite{wainwright-jordan-variational})
consists of iterating the \emph{mean-field equations} to search for a fixed point. The mean field equations are just the first-order optimality conditions for $\mathcal{F}^*$:
\begin{equation}\label{eqn:mean-field-equations}
x^* = \tanh^{\otimes n}(2Jx^* + h).
\end{equation}
In the Dobrushin uniqueness regime,
we prove that this message passing algorithm in fact converges exponentially fast to the optimum of the variational free energy.
\begin{theorem}\label{thm:message-passing}
Suppose $J$ is the interaction matrix of an Ising model with
arbitrary external field $h_i$ at vertex $i$, and suppose that
for every row $i$ of $J$, we have $\sum_{j} 2|J_{ij}| \le 1 - \eta$ 
for some uniform $\eta > 0$.
Let $x^*$ be the optimizer of the optimization problem given by $\mathcal{F}^*$.
Let $x_0$ be an arbitrary point in $[-1,1]^n$ and iteratively define
\[ x_n := \tanh^{\otimes n}(2 J x_{n - 1} + h). \]
Then,
\[ \|x_n - x^*\|_{\infty} \le (1 - \eta)^n \|x_0 - x^*\|_{\infty} \le 2 (1 - \eta)^n. \]
\end{theorem}
\begin{remark}\label{rmk:exp-slow}
The high-temperature assumption is necessary for this algorithm to converge
quickly to the optimum. In the super-critical Curie Weiss model without external field (\cref{example:curie-weiss} with $\beta > 1$ and $h = 0$), we see that
$x = (0,\ldots,0)$ is a critical point for the variational free energy (fixed point of the mean-field equations) but not the global optimum. Furthermore,
even if we start from the point $(\epsilon, \ldots, \epsilon)$ for $\epsilon$ a small positive number, we see that for $\beta$ large, iterating the mean field equations converges exponentially slowly in $\beta$ as $\tanh'(\beta)$ is exponentially small in $\beta$. 
\end{remark}
Even though there exist such situations where the optimization problem
defining the variational free energy is \emph{non-convex} and the message passing algorithm may fail or converge exponentially slowly (see \cref{rmk:exp-slow}),
there is a way to solve the optimization problem in polynomial time as long
as the model is ferromagnetic.
\begin{theorem}\label{thm:ferromagnetic-algorithm}
Fix an Ising model $J$ 
on $n$ vertices which is 
ferromagnetic (i.e. $J_{ij} \ge 0$ for every $i,j$) and has uniform external field $h$ at every node.
There is a randomized algorithm which runs in time $poly(n,1/\epsilon,\log(1/\delta))$ and succeeds with probability at least $1 - \delta$ in solving the optimization problem defining $\mathcal{F}^*$ up to $\epsilon$-additive error.
\end{theorem}
However, in the general case, we show that it is NP-hard to estimate
the variational free energy. In fact,
it is NP-hard to return an estimate to the 
free energy within additive error $n^{1 - \delta} \|J\|_F^{1 - \delta}$, 
whereas by \cref{thm-main-structural-result} and \cref{thm-mrf-main-structural-result}, the true variational free energy is much closer than this. 
\begin{theorem}\label{thm:free-energy-hardness}
For any fixed $\delta > 0$, it is NP-hard to approximate the free energy
$\mathcal{F}$ (or variational free energy $\mathcal{F}^*$) of an Ising model $J$ within an additive error of $n^{1 - \delta} \|J\|_F^{1 - \delta}$. More generally, for
an $r$-uniform Markov Random Field, it is NP-hard to approximate $\mathcal{F}$ within
an additive error of $(n^{r/2} \|J_{=r}\|_F)^{1 - \delta}$.
\end{theorem}
We now give an algorithm to approximate
the free energy in the most general setting;
in light of the NP-hardness result (\cref{thm:free-energy-hardness})
this approximation must be roughly on the scale of $n \|J\|_F$.
In the general setting, the only previous algorithm which gives non-trivial guarantees for approximating the log-partition function is that of Risteski \cite{risteski-ising}, which requires time $n^{O(1/\epsilon^2)}$ as well as stronger density assumptions in order to provide guarantee similar to \cref{thm:regularity-alg}. In comparison, the algorithm we give has the advantage that it runs in \emph{constant-time} for fixed $\epsilon$.
\begin{theorem}\label{thm:regularity-alg}
Fix $\epsilon > 0$. There is an algorithm which runs in time $2^{O(\log(1/\epsilon)/\epsilon^2)}$ and
returns, with probability at least $0.99$, an implicit description of a product distribution $\mu$
and estimate to the free energy $\hat{\mathcal{F}}$
such that
\[ \KL(\mu || P) \le \epsilon n \|J\|_F  + C\log(1/\epsilon)/\epsilon^2 + 0.5^{2^{1/\epsilon^2}} n \]
and
\[ |\mathcal{F} - \hat{\mathcal{F}}| \le \epsilon n \|J\|_F  + C'\log(1/\epsilon)/\epsilon^2 + 0.5^{2^{1/\epsilon^2}} n, \]
where $C$ and $C'$ are absolute constants. 
\end{theorem}
\begin{remark}
Typically, the first term in the bound of \cref{thm:regularity-alg} dominates. In particular, the last term $0.5^{2^{1/\epsilon^2}} n$ is dominated by the first term except in a very unusual regime where $\|J\|_F$ is very small i.e. the interactions in our model are extremely weak, and even then, it vanishes doubly-exponentially fast as we take $\epsilon \to 0$.
\end{remark}

Our algorithm extends in a straightforward way to general order $r$ Markov random fields as well. 

\begin{theorem}
\label{thm:algo-mrf-regularity-bad-dependence}
Fix $r \geq 3$. Then, there exists a constant $C=C(r)$  such that for any order $r$ Markov random field $J$ with Boltzmann distribution $P$ and free energy $\F$, and for any $\epsilon >0$, there is an algorithm which runs in time $2^{O(\log(1/\epsilon)/\epsilon^{2r-2})}$ and
returns, with probability at least $0.99$, an implicit description of a product distribution $\mu$
and estimate to the free energy $\hat{\mathcal{F}}$
such that
\[ \KL(\mu || P) \le \max_{1\leq d\leq r} \epsilon n^{d/2} \|J_{=d}\|_F  + C\log(1/\epsilon)/\epsilon^{2d-2} + 0.5^{2^{1/\epsilon^{2d-2} }} n \]
and
\[ |\mathcal{F} - \hat{\mathcal{F}}| \le \max_{1\leq d\leq r} \epsilon n^{d/2} \|J_{=d}\|_F  + C\log(1/\epsilon)/\epsilon^{2d-2} + 0.5^{2^{1/\epsilon^{2d-2} }} n. \]
\end{theorem}

In the previous theorem, it is possible to improve the dependence on $\epsilon$ at the expense of introducing a factor of $n^r$ in the running time. 
\begin{theorem}
\label{thm:algo-mrf-dependence-on-n}
Fix $r \geq 3$. Then, there exists a constant $C=C(r)$ such that for any order $r$ Markov random field $J$ with Boltzmann distribution $P$ and free energy $\F$, and for any $\epsilon > 0$, there is an algorithm which runs in time $2^{O(\log(1/\epsilon)/\epsilon^2)}n^{r}$ and
returns, with probability at least $0.99$, an implicit description of a product distribution $\mu$
and estimate to the free energy $\hat{\mathcal{F}}$
such that
\[ \KL(\mu || P) \le \epsilon \max_{1\leq d\leq r}n^{d/2} \|J_{=d}\|_F  + C\log(1/\epsilon)/\epsilon^2 + 0.5^{2^{1/\epsilon^2}} n \]
and
\[ |\mathcal{F} - \hat{\mathcal{F}}| \le \epsilon \max_{1\leq d\leq r}n^{d/2} \|J_{=d}\|_F  + C\log(1/\epsilon)/\epsilon^2 + 0.5^{2^{1/\epsilon^2}} n. \]
\end{theorem}

\subsection{Comparison with previous work}
\label{sec:previous-results}
As mentioned earlier, providing guarantees on the quality of the mean-field approximation for general graphs has attracted much interest in recent years. Notably, in the
context of graphons \cite{borgs2012convergent}, the following result (stated here in our notation\footnote{In their paper, the edge weights are normalized by $1/n$ so that on dense graphs, the limit as $n \to \infty$ will sensibly converge. Their bound is stated for the slightly more general setting of models over finite alphabets -- to facilitate ease of comparison, we have stated it only in the simplest case of binary Ising models with uniform external field $h$.}) was shown:
\[ |\mathcal{F}^*/n - \mathcal{F}/n| \le \frac{48}{n^{1/4}} + \frac{130 n \|\vec J\|_{\infty}}{\sqrt{\log n}} + \frac{5 |h|}{n^{1/2}}. \]
Here, $\|\vec{J}\|_{\infty}$ denotes the absolute value of the largest entry of $J$. 

This result was sufficient for the application 
in \cite{borgs2012convergent}, i.e., proving convergence of the free energy
density and the variational free energy density for sequences of dense graphs (i.e. those with $\Theta(n^2)$
many edges). In this case, it is natural to take $\|\vec J\|_{\infty} = O(1/n)$ and thus,
their error bound converges to 0 at rate $1/\sqrt{\log n}$. 
They used this
bound to prove that defining the free energy density of a graphon in terms
of the variational free energy density is asymptotically consistent with the combinatorial
definition of the free energy in terms of sums over states (which cannot naively be made sense of in the
graphon setting).

The bound in \cite{borgs2012convergent}
has two limitations: first, it does not provide any information about models where $\|\vec{J}\|_{\infty} = \omega(\sqrt{\log n}/n)$ -- a setting which includes
essentially all natural models on graphs with $o(n^2)$ edges --
and secondly, the convergence rate of $1/\sqrt{\log n}$ is very slow -- in order to get $\epsilon$
error in the bound, we must look at graphs of size $2^{1/\epsilon^2}$, which raises
the possibility that the approximation may perform badly even on very large graphs. 

The papers \cite{borgs2014p}, and most recently \cite{basak2017universality}, resolve the first issue by giving bounds which extend to sparser graphs. In our context, the latter result is more relevant, and we refer the reader to the discussion in \cite{basak2017universality}
for the relationship to \cite{borgs2014p}. The main result of \cite{basak2017universality}
is that $|\F^*/n - \F/n| = o(1)$ whenever $\|J\|_F^2 = o(n^2)$. 
As noted by the authors, if we do not
care about the rate of convergence, then this result is tight -- there are simple
examples of models with $\|J\|_F^2 = \Theta(n^2)$ where $|\F^*/n - \F/n| = \Omega(1)$.
However, their result is focused
on the asymptotic regime and does not give any control on the rate of convergence.
In contrast, our main result gives an explicit bound on the rate of convergence
which is optimal up to logarithmic factors (\cref{thm:variational-lb}). Moreover, this bound is much better than the one in \cite{borgs2012convergent}, even in regimes where the latter is applicable. For instance, in the setting of dense graphs with edge weights scaled by $1/n$, their bound shows that $|\mathcal{F}^*/n - \mathcal{F}/n|$ converges to $0$ at the rate $O(1/\sqrt{\log n})$, whereas our bound gives the convergence rate $O(\log^{1/3}(n)/n^{1/3})$.

It is interesting to note that both our result 
and \cite{borgs2012convergent} use the Frieze-Kannan weak regularity lemma. However, our analysis introduces a number of new ideas that let us avoid the $2^{1/\epsilon^2}$ dependence which is typical in applications of the weak regularity lemma, thereby obtaining bounds with exponentially better dependence on $n$. Besides giving the best known convergence rate, our result is almost as strong as \cite{basak2017universality} asymptotically and has a much
simpler proof which generalizes easily to higher-order Markov random fields. In contrast, the spectral methods used in \cite{basak2017universality}
may be more difficult to generalize to the case where higher-order tensors become involved.

With respect to \cref{thm:message-passing}, we note that some related ideas have been used in the convergence analysis of loopy belief propagation, which is a different algorithm unrelated to the mean field approximation (see for example \cite{TatikondaJordan:02,mooij-kappen}).

As far as algorithmic results are concerned, there has been a very long line of work historically in understanding the performance
of Markov Chain Monte Carlo methods (MCMC), especially the Glauber chain (Gibbs sampling). As mentioned earlier, it is known from \cite{dobrushin-uniqueness} that the Glauber dynamics
mix rapidly in the Dobrushin uniqueness regime, where the entries of each row of $J$ are bounded by $(1 - \eta)/2$. There has been
a lot of work on improvements to this result, see for example
\cite{mossel-sly} for a tight result on bounded degree graphs.
Although the Glauber dynamics typically cannot mix rapidly in the low-temperature regime (see e.g. \cite{sly-sun}), in the special case where $J$ is ferromagnetic, there is a different Markov chain which can approximately sample from the Boltzmann distribution in polynomial time \cite{JerrumSinclair:90}. Another result in the ferromagnetic regime using entirely different (deterministic) methods was given recently in \cite{lss-deterministic}. 

Note that in situations where Markov chain methods do work, they allow for approximate sampling and approximation of the partition
function to a higher precision than our results. However, in the general case where Markov chain methods typically have no guarantees, the previous best result is due to \cite{risteski-ising}, which gave a similar
guarantee for approximating the free energy as our \cref{thm:regularity-alg}, but requiring stronger
density assumptions as well as $n^{O(1/\epsilon^2)}$ time. It is interesting to note that this algorithm is essentially a variational method which works by taking a relaxation
of \cref{eqn:free-energy-variational-char} to \emph{pseudo-distributions} and giving a rounding scheme to convert
pseudo-distributions back to true probability distributions. However, the distributions produced by
the rounding process are more complicated than product distributions.

\begin{remark}
We finally note a recent preprint by the authors titled ``Approximating Partition Functions in Constant Time'' \cite{old-paper}.
\cite{old-paper} is completely superseded by this paper and \cite{next-paper}. The main focus of \cite{next-paper} is the sampling complexity of approximating the free energy of Ising models. Both the current paper and \cite{next-paper} include important references that the authors were not aware while writing \cite{old-paper}. 
\end{remark}

\subsection{Outline of the techniques}

The proof of our main structural inequality is based on the weak regularity lemma of Frieze and Kannan (\cref{fk}). Roughly speaking, this lemma allows us to (efficiently) partition the underlying weighted graph into a small number of blocks in a manner such that ``cut-like'' quantities associated to the graph approximately depend only on the \emph{numbers} of edges between various blocks. It is well known (see, e.g., \cite{borgs2012convergent}, and also \cref{lemma: free-energy-lipschitz}) that the free energy and variational free energy fit into this framework. This observation shows that in order to prove \cref{thm-main-structural-result}, it is sufficient to prove the statement for such graphs composed of a small number of blocks.

In order to do this, we will first show that the free energy for such graphs is well approximated by an intermediate optimization problem (\cref{eqn:intermediate-var-problem}) which is quite similar to the one defining the variational free energy. Next, we will use basic facts about entropy to show that the solution to this optimization problem is indeed close to the variational energy (\cref{lemma:epsilon-bound}). We now describe this intermediate optimization problem. 

The key point in the weak regularity lemma is that the number of blocks depends only on the desired quality of approximation, and \emph{not} of the size of the underlying graph. Since we only care about the numbers of edges between the various blocks, this allows us to approximately rewrite the sum computing the partition function in terms of only \emph{polynomially} many nonnegative summands, as opposed to the \emph{exponentially} many nonnegative summands we started out with (\cref{eqn:partition-function-cut}). Moreover, since none of the edge weights coming from the weak regularity lemma are too big, one can further group terms to reduce the number of summands to a polynomial in only the error parameter, independent of the number of vertices in the original graph (\cref{lemma:gamma-def}). This provides the desired intermediate optimization problem -- the log of the largest summand of this much smaller sum approximates the free energy  well (\cref{eqn:approx-sum-by-max-lb}, \cref{eqn:approx-sum-by-max-ub}). 

For the proof of \cref{thm:regularity-alg}, we show that solving (a slight variation of) this intermediate optimization problem amounts to solving a number of convex programs. However, since we want to provide algorithms which run in constant time (see \cref{rmk:constant-time-assumptions}), we first need to rewrite these programs in a manner which uses only a constant number of variables and constraints. The proofs of the corresponding theorems for general order $r$ Markov random fields follow a similar outline, with the application of \cref{fk} replaced by \cref{reg-alon-etal-mrf} or \cref{reg-fk-higher}.   

\subsection{Acknowledgements}
We thank David Gamarnik for insightful comments, Andrej Risteski for helpful discussions related to his work \cite{risteski-ising}, and Yufei Zhao for introducing us to reference \cite{alon-etal-samplingCSP-conference}.

\section{Preliminaries}

We will make essential use of the weak regularity lemma of Frieze and Kannan \cite{frieze-kannan-matrix}. Before stating
it, we introduce some terminology. Throughout this section, we will
deal with $m\times n$ matrices whose entries we will index by $[m]\times[n]$,
where $[k]=\{1,\dots,k\}$. 
\begin{defn}
Given $S\subseteq[m]$, $T\subseteq[n]$ and $d\in\R$, we define
the $[m]\times[n]$ \emph{Cut Matrix }$C=CUT(S,T,d)$ by 
\[
C(i,j)=\begin{cases}
d & \text{if }(i,j)\in S\times T\\
0 & \text{otherwise}
\end{cases}
\]
\end{defn}

\begin{defn}
A \emph{Cut Decomposition }expresses a matrix $J$ as 
\[
J=D^{(1)}+\dots+D^{(s)}+W
\]

where $D^{(i)}=CUT(R_{i},C_{i},d_{i})$ for all $t=1,\dots,s$. 
We say that such a cut decomposition has \emph{width }$s$\emph{,
coefficient length $(d_{1}^{2}+\dots+d_{s}^{2})^{1/2}$ }and \emph{error
$\|W\|_{\infty\mapsto1}$}.
\end{defn}

We are now ready to state the weak regularity lemma of Frieze and Kannan \cite{frieze-kannan-matrix}. The particular choice of constants can be found in \cite{alon-etal-samplingCSP-conference}. 
\begin{theorem}
\label{fk}
\cite{frieze-kannan-matrix}
Let $J$ be an arbitrary real matrix, and let $\epsilon>0$.
Then, we can find a cut decomposition of width at most $16/\epsilon^{2}$, 
coefficient length at most $4\|J\|_F/\sqrt{mn}$, error at most $4\epsilon\sqrt{mn}\|J\|_F$, and such that $\|W\|_{F}\leq\|J\|_F$.  
\end{theorem}

\section{Proof of the main structural result}
We begin by showing that both the free energy and the variational free energy are $1$-Lipschitz with respect to the cut norm of the matrix of interaction strengths. 

\begin{lemma}
\label{lemma: free-energy-lipschitz} 
Let $J$ and $D$ be the matrices of interaction strengths of 
Ising models with partition functions $Z$ and $Z_{D}$, and variational free energies $\F^{*}$ and $\F^{*}_{D}$. Then, with $W:= J - D$, we have  
$|\log Z-\log Z_{D}|\leq \|W\|_{\infty \mapsto 1}$ and $|\F^{*}-\F^{*}_{D}|\leq\|W\|_{\infty \mapsto 1}$. 
\end{lemma}
\begin{proof}
Note that for any $x\in[-1,1]^{n}$, we have 
\begin{align*}
|\sum_{i,j}J_{i,j}x_{i}x_{j}-\sum_{i,j}D_{i,j}x_{i}x_{j}| & =|\sum_{i}(\sum_{j}W_{i,j}x_{j})x_{i}| \leq|\sum_{i}|\sum_{j}W_{i,j}x_{j}|\\
 & \leq\|W\|_{\infty\mapsto1}, 
\end{align*}
from which we immediately get that $|\F^{*}-\F^{*}_{D}|\leq\|W\|_{\infty \mapsto 1}$. 
Moreover, for any $x\in\{\pm1\}^{n}$, we have 
\[
\exp(\sum_{i,j}J_{i,j}x_{i}x_{j}) \in \left[ \exp\left(\sum_{i,j}D_{i,j}x_{i}x_{j}) \pm \|W\|_{\infty \mapsto 1}\right) \right].
\]
Taking first the sum of these inequalities over all $x\in\{\pm1\}^{n}$
and then the log, we get 
\[
\log Z \in \left[ \log\left(\sum_{x\in\{\pm1\}^{n}}\exp \left(x^{T}Dx \right)\right) \pm \|W\|_{\infty \mapsto 1} \right],
\]
as desired. 
\end{proof}
\begin{remark}
\label{rmk:applying-regularity-structural-result}
For the remainder of this section, we take $D:=D^{(1)}+\dots+D^{(s)}$, where $D^{1},\dots,D^{s}$ are the cut matrices coming from applying \cref{fk} to $J$ with parameter $\epsilon/12$, so that $s \le 2304/\epsilon^2$ and $\|J - W\|_{\infty \mapsto 1} \le \|J\|_F/3$. By \cref{lemma: free-energy-lipschitz}, it follows that $|\log Z -\log Z_{D}|\leq \epsilon n\|J\|_F/3$ and $|\F^{*}-\F^{*}_{D}|\leq \epsilon n\|J\|_F/3 $. Thus, in order to show that $\F - \F^{*} \leq \epsilon n\|J\|_F$, it suffices to show that $\log Z_D - \F^{*}_D \leq \epsilon n\|J\|_F/3$. 
\end{remark}

In order to show this, we begin by approximating $\log Z_{D}$ by the solution to an optimization problem. Let 
$R_{i}$ (resp. $C_{i}$) denote the rows (respectively columns) corresponding
to the cut matrix $D^{(i)}$. Then,
it follows by definition that

\[
Z_{D}=\sum_{x\in\{\pm1\}^{n}}\exp\left(\sum_{i=1}^{s}r_{i}(x)c_{i}(x)d_{i}\right),
\]
where $r_{i}(x)=\sum_{a\in R_{i}}x_{a}$ and $c_{i}(x)=\sum_{b\in C_{i}}x_{b}$.
By rewriting the sum in terms of the possible values that $r_{i}(x)$
and $c_{i}(x)$ can take, we get that 
\begin{equation}
\label{eqn:partition-function-cut}
Z_{D}=\sum_{r,c}\exp\left(\sum_{i=1}^{s}r_{i}(x)c_{i}(x)d_{i}\right)\left(\sum_{x\in\{\pm1\}^{n}:r(x)=r,c(x)=c}1\right),
\end{equation}
where $r=(r_{1},\dots,r_{s})$ ranges over all elements of $[-|R_{1}|,|R_{1}|]\times\dots\times[-|R_{s}|,|R_{s}|]$
and similarly for $c$. The following lemma shows that for estimating the contribution of the term corresponding to some vector $x$, it suffices to know the components of $x$ up to some constant precision. 
\begin{lemma}\label{lemma:gamma-def}
Let $J, D^{1},\dots,D^{s}$ be as above. Then, given real numbers $r_{i},r'_{i},c_{i},c'_{i}$ for each $i\in[s]$ and some 
$\upsilon \in (0,1)$ such that $|r_i|,|c_i|,|r'_i|,|c'_i| \le n$, 
$|r_i - r'_i| \le \upsilon n$ and $|c_i - c'_i| \le \upsilon n$ 
for all $i\in[s]$, we get that 
$\sum_i d_i|r'_i c'_i - r_i c_i| \le 8\|J\|_F\upsilon ns^{1/2}$. 
\end{lemma}
\begin{proof}
Since $ |r'_i c'_i - r_i c_i| \le |c'_i||r'_i - r_i| + |r_i||c'_i - c_i| \le 2\upsilon n^2$, it follows by Cauchy-Schwarz that
\begin{align*}
\sum_i d_i|r'_i c'_i - r_i c_i|
\le \left(\sum_i d_i^2\right)^{1/2} 2s^{1/2}\upsilon n^2
\le  8\|J\|_F\upsilon ns^{1/2}.
\end{align*} 
\end{proof}

The previous lemma motivates grouping together configurations $x$ with similar values of $r_i(x),c_i(x)$. Accordingly, for any $r \in [-|R_1|,|R_1|]\times\dots\times[-|R_s|,|R_s|]$, $c \in [-|C_1|,|C_1|]\times\dots\times[-|C_s|,|C_s|]$ and $\upsilon > 0$, let 
$$X_{r,c,\upsilon}:= \{x\in \{\pm1\}^{n}: |r_i(x)-r_i| \leq \upsilon n, |c_i(x)-c_i| \leq \upsilon n \text{ for all } i\in [s]\}.$$ 
Let $I_\upsilon := \{\pm\upsilon n, \pm 3\upsilon n,\pm 5\upsilon n,\dots,\pm\ell \upsilon n\}$, where $\ell$ is the smallest odd integer satisfying $|\ell \upsilon n - n| \leq \upsilon n$, so $|I_{\upsilon}| \le 1/\upsilon + 1$. Let 
\begin{equation}
\label{eqn:intermediate-var-problem}
Z_{D,\upsilon}^{*}:=\max_{r,c\in I_{\upsilon}^{s}}\exp\left(\sum_{i=1}^{s}r_{i}c_{i}d_{i}+\log|X_{r,c,\upsilon}|\right).
\end{equation}

Then, it follows immediately from \cref{lemma:gamma-def} that 
\[
Z^{*}_{D,\upsilon}\exp\left(-8\|J\|_F\upsilon ns^{1/2}\right) \leq Z_{D}\leq\sum_{r,c\in I_{\upsilon}^{s}}|X_{r,c,\upsilon}|\exp\left(\sum_{i=1}^{s}r_{i}c_{i}d_{i}\right)\exp\left(8\|J\|_F\upsilon ns^{1/2}\right).
\]

In particular, since the outer sum is over $|I_{\upsilon}|^{2s}$ terms, it follows
that
\begin{equation}
\label{eqn:approx-sum-by-max-lb}
\log Z_{D,\upsilon}^{*}\geq\log Z_{D}-8\|J\|_F\upsilon ns^{1/2} - 2s\log |I_\upsilon| \ge \log Z_{D}-8\|J\|_F\upsilon ns^{1/2} - 2s\log(1/\upsilon + 1)
\end{equation}
and
\begin{equation}
\label{eqn:approx-sum-by-max-ub}
\log Z^{*}_{D,\upsilon}\leq\log Z_{D}+8\|J\|_F\upsilon ns^{1/2}.
\end{equation}

We can now prove \cref{thm-main-structural-result}:
 all we need to do is give an upper bound on $\F - \F^{*}$.
\begin{lemma}\label{lemma:epsilon-bound}
For any $\epsilon > 0$,
\[ \F - \mathcal{F^*} \le \epsilon n \|J\|_F + 10^5 \log(e + 1/\epsilon)/\epsilon^2. \]
\end{lemma}
\begin{proof}
\label{proof:main-structural-result}
Let $\gamma=\epsilon/48s^{1/2}$. 
Let $r=(r_1,\dots,r_s),c=(c_1,\dots,c_s) \in I_{\gamma}^{s}$ be such that $\log Z^{*}_{D,\gamma}=\sum_{i=1}^{n}r_{i}c_{i}d_{i}+\log|X_{r,c,\gamma}|$. 
Define
\[ \bar{y}_j := \frac{1}{|X_{r,c,\gamma}|} \sum_{x \in X_{r,c,\gamma}} x_j, \]
and let $Y := (Y_1, \ldots, Y_n)$ be a random vector distributed uniformly in 
$X_{r,c,\gamma}$. Then by the chain rule for entropy,
\[ \log |X_{r,c,\gamma}| = H(Y) \le \sum_{j = 1}^n H(Y_j) = \sum_{j = 1}^n H\left(\frac{1 + \overline{y}_j}{2}\right). \]
Using this, we have 
\begin{align*}
\log Z_{D,\gamma}^{*} & =\sum_{i=1}^{n}r_{i}c_{i}d_{i}+\log|X_{r,c,\gamma}|\\
 & \leq\sum_{i=1}^{n}r_{i}c_{i}d_{i}+\sum_{j = 1}^n H\left(\frac{1 + \overline{y}_j}{2}\right)\\
 & \leq\left\{ \sum_{i=1}^{n}r_{i}(\bar{y})c_{i}(\bar{y})d_{i}+8\|J\|_F\gamma ns^{1/2}\right\} + \sum_{j = 1}^n H\left(\frac{1 + \overline{y}_j}{2}\right)\\
 & =\left\{ \sum_{i=1}^{n}r_{i}(\bar{y})c_{i}(\bar{y})d_{i}+\sum_{j = 1}^n H\left(\frac{1 + \overline{y}_j}{2}\right)\right\} +8\|J\|_F\gamma ns^{1/2}\\
 & \leq \F_{D}^{*}+8\|J\|_F\gamma ns^{1/2},
\end{align*}
where the third line follows from $\overline{y}$ lying in the convex hull of $X_{r,c,\gamma}$ and \cref{lemma:gamma-def}, and the last line follows from the definition of $\F^{*}_{D}$. Thus, we get
\begin{align*}
\F_{D}^{*} & \geq \log Z_{D,\gamma}^{*} - 8\|J\|_F\gamma ns^{1/2}\\
 & \geq \log Z_{D} - 16\|J\|_F\gamma ns^{1/2} - 2s\log(1/\gamma + 1)\\
 & \geq \log Z_{D} - \frac{\epsilon n\|J\|_F}{3} - 10^5\log\left(\frac{1}{\epsilon} + e\right)\frac{1}{\epsilon^{2}},
\end{align*}
where in the second line, we have used \cref{eqn:approx-sum-by-max-lb}, and in the last line, we have used the values of $\gamma$ and $s$. 
Now, \cref{rmk:applying-regularity-structural-result} gives
\[ \mathcal{F} - \mathcal{F}^* \le \epsilon n \|J\|_F + 10^5 \log(e + 1/\epsilon)/\epsilon^2. \]
\end{proof}
Finally, we use this bound to prove Theorem~\ref{thm-main-structural-result}.

\begin{proof}[Proof of \cref{thm-main-structural-result}]
Fix $M > e$ a constant to be optimized later. 
Observe that since $\E_{\mu}[\sum_{i,j} J_{i,j}X_i X_j] \leq n \|J\|_{F}$ by Cauchy-Schwartz, and since $\F^* \geq n$, we always have $\F - \F^* \leq n\|J\|_F$. Therefore, if $n \|J\|_F \le M$, we see that $\mathcal{F} - \mathcal{F}^* \le n \|J\|_F \le M^{1/3} (n \|J\|_F)^{2/3}$.

Next, we analyze the case when $n \|J\|_F > M$.
Taking $\epsilon = \left(\frac{M\log(n \|J\|_F + e)}{n \|J\|_F \log{M}}\right)^{1/3}$ in \cref{lemma:epsilon-bound} gives
\begin{align*}
\mathcal{F} - \mathcal{F}^* 
&\le (M/\log M)^{1/3} n^{2/3} \|J\|_F^{2/3} \log^{1/3}(n \|J\|_F + e)  + 10^5(\log M / M)^{2/3} \frac{\log(n \|J\|_F + e)}{\log^{2/3}(n \|J\|_F + e)} n^{2/3} \|J\|_F^{2/3} \\
&\le \left((M/\log M)^{1/3} + 10^5 (\log M/M)^{2/3}\right) n^{2/3} \|J\|_F^{2/3} \log^{1/3}(n \|J\|_F + e).
\end{align*}

Finally, taking $(M/\log M) = 10^{5}$, we see that for all values of $n \|J\|_F$,
\[ \mathcal{F} - \mathcal{F}^* \le 200 n^{2/3} \|J\|_F^{2/3} \log^{1/3}(n \|J\|_F + e). \]
\end{proof}

\begin{proof}[Proof of \cref{thm-mrf-main-structural-result}]
The proof is exactly the same as that of Theorem~\ref{thm-main-structural-result},
except that for each $d$ from $1$ to $r$, we use the following generalized
weak regularity lemma to decompose $J_{=d}$:
\begin{theorem}
\cite{alon-etal-samplingCSP-journal}\label{reg-alon-etal-mrf}
Let $J$ be an arbitrary $k$-dimensional matrix on $X_{1}\times\dots\times X_{k}$,
where we assume that $k\geq 1$ is fixed. Let $N:=|X_{1}|\times\dots\times|X_{k}|$
and let $\epsilon>0$. Then, in time $2^{O(1/\epsilon^{2})}O(N)$
and with probability at least $0.99$, we can find a cut decomposition of
width at most $4/\epsilon^{2}$, error at most $\epsilon\sqrt{N}\|J\|_F$,
and the following modified bound on coefficient length: $\sum_i |d_i| \le 2\|J\|_F/\epsilon\sqrt{N}$, where $(d_i)_{i =1}^s$ are the coefficients of
the cut arrays.
\end{theorem}

We omit further details. 
\end{proof}

\section{An almost matching lower bound for a large class of variational methods}
\begin{proof}[Proof of \cref{thm:variational-lb}]
Let $(\mathcal{Q}_n)_{n = 0}^{\infty}$ be a sequence of families
of probability distributions as in the theorem statement. By assumption, there exist $k$ and $J$ such that $\mathcal{Q}_k$ does not contain the probability distribution $P_J$ corresponding to the Ising model $J$ on $k$ nodes. We denote by $Q_J$ the probability distribution in $\mathcal{Q}_k$ which is closest to $P_J$. In particular, by the closure under products assumption, we have that $Q_J^{\otimes m} \in \mathcal{Q}_{mk}$ for all integers $m \geq 1$.

Consider the Ising model on $n:= mk$ nodes whose matrix of interaction strengths $J'_n$ is the block diagonal matrix consisting of $m$ copies of $J$. Combinatorially, we can view $J'_n$ as $m$ vertex disjoint copies of $J$. 
We claim that $Q_J ^{\otimes m}$ is the
closest distribution in $\mathcal{Q}_{mk}$ to the Ising model $J'$. Suppose on the contrary that there is some other distribution $Q_{J'} \in \mathcal{Q}_{mk}$ which is strictly closer to $P_{J'}$ than $Q_J ^{\otimes m}$. Then, the chain rule for KL divergence immediately implies that there exists some distribution $\tilde{Q}_J$ on $\{\pm 1\}^{k}$, obtained by conditioning $Q_{J'}$ on $k(m-1)$ variables, which is strictly closer to $P_J$ than $Q_J$. Since $\tilde{Q}_J \in \mathcal{Q}_k$ by assumption, and since $Q_J$ is the closest distribution to $P_J$ in this class, this gives a contradiction. 

Therefore, we see that  
$$ \inf_{Q\in \mathcal{Q}_n}\KL(Q|| P_{J'}) \geq m \KL (Q_J|| P_J) = \Theta(n).$$
Furthermore, $\|J\|_F = \Theta(\sqrt{n})$ so that $n^{2/3} \|J\|_F^{2/3} = \Theta(n)$. Hence, we see that the variational method corresponding to $(\mathcal{Q}_n)_{n = 0}^{\infty}$ must make an error of size $\Omega(n^{2/3} \|J\|_F^{2/3}$).
\end{proof}

\section{The high-temperature regime}
In this section, we show that in the high-temperature regime
where Markov chain methods are guaranteed to mix quickly, the
variational free energy functional is convex and furthermore,
a simple message passing algorithm solves the corresponding
optimization problem quickly.
\begin{lemma}\label{lem:H-strong-concave}
For $H(p) := H(Ber(p))$ and for any $p \in [0,1]$, we have
\[ H''(p) \le -4. \]
\end{lemma}
\begin{proof}
By definition,
\[ H(p) = -p \log p - (1 - p)\log(1 - p). \]
Therefore,
\[ H'(p) = -\log p - 1 + \log(1 - p) + 1 = -\log p + \log(1 - p), \]
and
\[ H''(p) = -\frac{1}{p} - \frac{1}{1 - p} \le -4. \]
\end{proof}
\begin{proof}[Proof of \cref{thm:high-temperature-convex}]
Recall that $J$ is symmetric and has diagonal entries $0$. Therefore, the assumption $\sum_{j} 2|J_{i,j}| \leq 1$ for all $i$, along with Gershgorin's disk theorem, shows that all the eigenvalues of $J$ lie in $[-1/2, 1/2]$. 
Observe that the Hessian
of the corresponding quadratic form is $2J$. Combining
this with the strong concavity of entropy (\cref{lem:H-strong-concave}) and the chain rule, which gives $\frac{d^2}{dx^2} H((1 + x)/2) \le -1$, proves the concavity claim. 

The runtime complexity
follows from standard algorithms from convex optimization,
e.g. standard guarantees for the ellipsoid method (\cite{gls}).
\end{proof}
\begin{proof}[Proof of \cref{thm:message-passing}]
Since $\tanh$ is $1$-Lipschitz, we have for any 
$x^1,x^2 \in [-1,1]^n$ that
\[ \|\tanh^{\otimes n}(2J x^1 + h) - \tanh^{\otimes n}(2J x^2 + h)\|_{\infty}
\le 2\|J x^1 - J x^2\|_{\infty}
\le (1 - \eta) \|x^1 - x^2\|_{\infty}. \]
Since the optimum $x^*$ is a fixed point of the mean field equations, the above inequality shows that 
\[ \|\tanh^{\otimes n}(2J x_{n + 1} + h) - x^*\|_{\infty} \le (1 - \eta)\|\tanh^{\otimes n}(2J x_n + h) - x^*\|_{\infty}, \]
and iterating this inequality gives the desired conclusion.
\end{proof}

\section{Computing the mean-field approximation in ferromagnetic models}
\begin{proof}[Proof of Theorem~\ref{thm:ferromagnetic-algorithm}]
Consider the $m$-blow up of the Ising model, denoted by $J_m$,  defined as follows: replace each vertex $i$ by $m$ vertices $(i,1),\dots,(i,m)$, add an edge of weight $J_{i,j}/m$ between vertices $(i,k)$ and $(j,\ell)$ for all $1\leq k,\ell \leq m$, and assign a uniform external field $h$ at each vertex $(i,k)$. 

Given a spin vector $X$ sampled from
the Boltzmann distribution of $J_m$, 
define $Y_i \in [-m,m]$ to be
the net spin of the vertices $(i,1),\dots, (i,m)$. 
Let $N_y$ denote the number of spin vectors which correspond
to the net spin vector $y$ via the correspondence above. Then, we see that
\begin{align*}
\Pr(Y = y) 
&= \frac{1}{Z}\exp\left(\sum_{i,j}\frac{J_{ij}y_iy_j}{m} + h \sum_i y_i + \log N_y\right) \\
&= \frac{1}{Z}\exp\left(m \sum_{i,j} J_{i,j} (y_i/m) (y_j/m) + m h \sum_i (y_i/m) + m \sum_i H\left(\frac{1 + y_i/m}{2}\right) \pm O(n\log m)\right).
\end{align*}

Let $Y_{\epsilon}$ be the set of $y$ such that
\[ \sum_{i,j} J_{ij} (y_i/m) (y_j/m) + h \sum_i (y_i/m) + \sum_i H\left(\frac{1 + y_i/m}{2}\right) < \mathcal{F}^* - \epsilon, \]
where $\mathcal{F}^*$ is the variational free energy
of the original Ising model $J$. Note that $Z \geq e^{\F^* - O(n\log{m})}$, as is readily seen by considering the net spin vector $y^*$ given by $y^*_i = mx^*_i$, where $x^* = (x^*_1,\dots, x^*_m)$ is the optimizer of the optimization problem defining $\F^*$.
Then, the above inequality shows that for each $y \in Y_{\epsilon}$,
\[ \Pr(Y = y) \le e^{-m \epsilon \pm O(n\log m)}.\]
Since $|Y_{\epsilon}| \le m^n$, the union bound shows that  
\[ \Pr(Y \in Y_{\epsilon}) \le e^{-m \epsilon \pm O(n\log m)} \leq \frac{1}{3},\]
provided we take $m = \Omega(n \log(n)/\epsilon)$. 

The preceding analysis shows the following: if we use the algorithm of Jerrum and Sinclair \cite{JerrumSinclair:90} to draw
$O(\log(1/\delta))$ independent (approximate) samples $X$ from the Boltzmann distribution of $J_m$, and use these (approximate) samples to obtain normalized net spin vectors $Y/m \in \{\pm 1\}^n$, then with probability $1 - \delta$, at least one of the sampled $Y/m$ solves the optimization problem defining $\F^*$ up to $\epsilon$-additive error in the objective.
\end{proof}
\begin{remark}
It is known by the result of \cite{Goldberg-Jerrum} that
approximate sampling becomes \#BIS-hard for ferromagnetic
Ising models if we allow different (inconsistent) external fields for each node. Thus, our algorithm does not extend to this setting.
\end{remark}

\section{NP-hardness: proof of \cref{thm:free-energy-hardness}}
Our proof is an easy consequence of hardness of approximation results for dense CSPs. Specifically, we rely on a hardness result
for \emph{fully dense} MAX-r-LIN-2. In this problem, we are given $n$ free variables $x_1, \ldots, x_n$ to be assigned values in $\mathbb{F}_2^n$. Moreover, for each of the ${n \choose r}$ subsets $S$ of $[n]$ of size $r$, we are given a constraint $\sum x_S \equiv y_S$ mod 2, for $y_S$ fixed to be either 0 or 1. The goal is to find the maximum number of constraints which can be satisfied simultaneously by a single assignment of $x_1, \ldots, x_n$. For reasons of convenience, the objective
value is defined to be $(1/2)(\text{\# of satisfied constraints}) - (1/2)(\text{\# of violated constraints})$.
\begin{theorem}[\cite{ailon2007hardness}]
For $r \geq 2$ and any $\epsilon >0$, it is NP-hard to approximate fully dense MAX-r-LIN-2 
within an
additive error of $n^{r - \epsilon}$. 
\end{theorem}

\begin{proof}[Proof of \cref{thm:free-energy-hardness}]
We illustrate the reduction to our problem in the case $r = 2$. Given
an instance of fully dense MAX-r-LIN-2 with constraints corresponding
to fixed $(y_S)_{|S| = 2}$, we consider the Ising model with matrix of interaction strengths $J$, where 
$J_{ij} = 1/2 - y_{\{i,j\}}$. It is readily seen that for any distribution $\mu$ on $\{\pm 1\}^n$, 
$$\sum_{i,j} J_{ij} \E_{\mu}[X_i X_j] \leq \text{MAX-r-LIN-2}(y). $$ 

On the other hand, denoting by $x = (x_1,\dots, x_n)$ the optimal assignment of the variables for MAX-r-LIN-2$(y)$, it is immediate that the deterministic distribution $\nu$ concentrated on $X_i = (-1)^{x_i}$ satisfies  
$$\sum_{i,j} J_{ij} \E_{\nu}[X_i X_j] = \text{MAX-r-LIN-2}(y).$$

Thus, it follows that 
\[ \F = \max_{\mu}\left[\sum_{i,j} J_{ij} \E_{\mu}[X_i X_j] + H(\mu)\right] = \text{MAX-r-LIN-2}(y) \pm n. \]

Finally, observe that $n \|J\|_F = \Theta(n^2)$. Therefore, approximating $\F$ within
additive error $(n \|J\|_F)^{1 - \delta}$ gives an $n^{2(1 - \delta)}$ additive approximation to MAX-r-LIN-2$(y)$.
\end{proof}
\section{A general algorithm for solving the variational problem}
By \cref{thm-main-structural-result}, we have an upper bound on $\KL(\nu || P)$ (which is almost tight in the worst case) for the optimal product distribution $\nu$. Unfortunately, \cref{thm:free-energy-hardness} shows that it is not always possible to efficiently find a product distribution which is as close to $P$ as $\nu$ is. In this section, we describe a provable algorithm which does essentially as well as possible without violating \cref{thm:free-energy-hardness}, with the additional benefit that it runs in \emph{constant time} (independent of the size of the graph). Here we will use $\tilde{O}$ notation to hide logarithmic factors independent of $n$.

\begin{remark}
\label{rmk:constant-time-assumptions}
In order to provide a constant time guarantee on problems with unbounded input size, we will work under the usual assumptions on the computational model for sub-linear algorithms (as in, e.g., \cite{alon-etal-samplingCSP-conference,frieze-kannan-matrix,indyk1999sublinear}). Thus, we can probe matrix entry $A(i,j)$ in $O(1)$ time. Note also that by the standard Chernoff bounds, it follows that for any set of vertices $V$ for which we can test membership in $O(1)$ time, we can also estimate $|V|/n$ to additive error $\epsilon$ w.h.p. in constant time using $\tilde{O}(1/\epsilon^2)$ samples. This approximation will always suffice for us and so, for the sake of exposition, we will henceforth ignore this technical detail and just assume that we have access to $|V|/n$ (as in, e.g., \cite{frieze-kannan-matrix}).

\end{remark}

\begin{proof}[Proof of \cref{thm:regularity-alg}]
Observe that it suffices to return the description of an $x \in [-1,1]^{n}$ such that $\F(x):= \sum_{i,j}J_{i,j}x_ix_j + \sum_i H((1+x_i)/2)$ satisfies 
$$\F^* \leq \F(x) + \frac{2\epsilon n}{3}\|J\|_{F} + 0.5^{2^{1/\epsilon^2}}n .$$ 
Indeed, since \[ \mathcal{F} - \mathcal{F^*} \le \frac{\epsilon n}{3} \|J\|_F + 10^6 \log(e+1/\epsilon)/\epsilon^2 \]
by \cref{lemma:epsilon-bound}, the product distribution $\mu$ for which the $i^{th}$ coordinate has expected value $x_i$ will then satisfy the conclusions of the theorem.  

Our strategy for finding such an $x$ will be to find an approximate maximizer of the problem defining $\F^*_{D}$, where $D$ is a sum of a small number of cut matrices which is close to $J$ in the $\|\cdot\|_{\infty \mapsto 1}$ norm. Specifically, we 
use the following algorithmic weak regularity lemma of Frieze and Kannan:
\begin{theorem}
\label{fk-algorithmic}
\cite{frieze-kannan-matrix}
Let $J$ be an arbitrary real matrix, and let $\epsilon,\delta>0$.
Then, in time $2^{\tilde{O}(1/\epsilon^{2})}/\delta^{2}$, we can,
with probability at least $1-\delta$, find a cut decomposition of width
$O(\epsilon^{-2})$, coefficient length at most $\sqrt{27}\|J\|_F/\sqrt{mn}$
and error at most $4\epsilon\sqrt{mn}\|J\|_F$. 
\end{theorem}
As in \cref{rmk:applying-regularity-structural-result}, we will take $D := D^{(1)}+\dots + D^{(s)}$, where the $D^{(i)}$ are obtained by applying \cref{fk-algorithmic} to $J$ with parameter $\epsilon/12$. In particular, \cref{lemma: free-energy-lipschitz} shows that $|\F^* - \F^*_{D}| \leq \frac{\epsilon n}{3}\|J\|_{F}$, so that any $x$ which satisfies $\F^*_{D} \leq \F(x) + \frac{\epsilon n}{3}\|J\|_{F} + 0.5^{2^{1/\epsilon^2}}n$ also satisfies the desired upper bound on $\F^* - \F(x)$.  

For $r,c \in I^{s}_{\gamma}$, where $I_{\gamma}$ is as in the proof of \cref{proof:main-structural-result}, consider the following max-entropy program $\mathcal{C}_{r,c,\gamma}$: 
\begin{align*}
\max & \quad \sum_{i=1}^{n}H\left(\frac{1+x_{i}}{2}\right)\\
s.t.\\
\forall i\in[n]: & \quad -1\leq x_{i}\leq1 \\
\forall t\in[s]: & \quad r_{t}-\gamma n\leq\sum_{i\in R_{t}}x_{i}\leq r_{t}+\gamma n\\
\forall t\in[s]: & \quad c_{t}-\gamma n\leq\sum_{i\in C_{t}}x_{i}\leq c_{t}+\gamma n
\end{align*}

Then, \cref{lemma:gamma-def} shows that  
\[ \overline{\F}_D := \max_{r,c \in I^s_v} \sum_{i = 1}^s r_i c_i d_i + \mathcal{C}_{r,c,\gamma} \]
satisfies $|\overline{\F}_{D} - \F^*_{D}| \leq \frac{\epsilon n}{3}\|J\|_{F}$, provided we take $\gamma \leq s^{-1/2}/24$. Let $(\overline{r},\overline{c})$ denote the values of $(r,c)$ attaining $\overline{\F}_{D}$. It follows that if we can return an $x \in [-1,1]^{n}$ such that $x$ is feasible for $\mathcal{C}_{\overline{r},\overline{c},\gamma}$, and 
$$\mathcal{C}_{\overline{r},\overline{c},\gamma} \leq \sum_{i}H\left(\frac{1+x_i}{2}\right) + 2^{-2^{1/\epsilon^2}}n,$$
then we would be done. 

Since we want our algorithm to run in constant time, we rewrite this convex program in an equivalent way with only a constant number of variables and constraints. Let $(V_a)_{a=1}^{A}$ denote the common refinement of $\{R_i, C_i\}_{i=1}^{s}$. In particular, note that $A \leq 2^{2s}$. 
Let $n v_a$ denote
the number of vertices in $V_a$, and recall (\cref{rmk:constant-time-assumptions}) that we can estimate $v_a$ to
high precision in constant time by sampling. Then, by the concavity of entropy, it is readily seen that for the maximum entropy program $\mathcal{H}_{r,c,\gamma}$:
\begin{alignat*}{4}
&\max\quad &\sum_a &v_a H\left(\frac{1 + z_a/v_a}{2}\right)\\
&\ s.t.\quad&-v_a &\le z_a &&\le v_a &\qquad& \forall 1 \le a \le A\\
&&{r_t/n} &\le \sum_{a : V_a \subset R_t} z_a &&\le {r_t/n} + \gamma && \forall 1 \le t \le s \\
&&{c_t/n} &\le \sum_{a : V_a \subset C_t} z_a &&\le {c_t/n} + \gamma && \forall 1 \le t \le s,
\end{alignat*}
we have $n\mathcal{H}_{r,c,\gamma} = \mathcal{C}_{r,c,\gamma}$. Finally, each of these convex programs can be solved approximately using standard guarantees for the ellipsoid method \cite{gls} -- in time $2^{O(1/\epsilon^2)}$,
the returned $z_a$ is optimal up to an additive error of $2^{-2^{1/\epsilon^2}}$, and this completes the proof. 
\end{proof}

\begin{proof}[Proofs of \cref{thm:algo-mrf-regularity-bad-dependence} and \cref{thm:algo-mrf-dependence-on-n}] The proofs of these theorems are essentially the same as the proof of \cref{thm:regularity-alg}, and therefore we will omit details. We only note that for \cref{thm:algo-mrf-regularity-bad-dependence}, we apply the following algorithmic regularity lemma of Frieze and Kannan generalizing \cref{fk-algorithmic}:
\begin{theorem}\cite{frieze-kannan-matrix}\label{reg-fk-higher}
Suppose $J$ is an arbitrary $k$-dimensional matrix on $X_{1}\times\dots\times X_{k}$,
where we assume that $k\geq3$ is fixed. Let $N:=|X_{1}|\times\dots\times|X_{k}|$
and let $\epsilon,\delta\in(0,1]$. Then, in time $O(k^{O(1)}\epsilon^{-O(\log_{2}k)}2^{\tilde{O}(1/\epsilon^{2})}\delta^{-2})$,
we can, with probability at least $1-\delta$, find a cut decomposition
of width $O(\epsilon^{2-2k})$, coefficient length at most $\sqrt{27}^{k}\|J\|_F/\sqrt{N}$
and error at most $\epsilon2^{k}\sqrt{N}\|J\|_F$.
\end{theorem}

For \cref{thm:algo-mrf-dependence-on-n}, we instead use \cref{reg-alon-etal-mrf}. 
\end{proof}
 
\bibliographystyle{plain}
\bibliography{ising-regularity,all}

\newpage 

\appendix

\section{An Improved Result for Graphs of Low Threshold-Rank}
\begin{defn}
Let $J$ be the matrix of interaction strengths of an Ising
model and define the \emph{degree} of a vertex $u$ to be
\[ d(u) = \sum_{v} |J_{uv}|.\]
Let $D = diag(d(u))$ be the matrix of degrees, then
the \emph{normalized adjacency matrix} $J_{\mathcal D}$ is given by
\[ J_{\mathcal D} = D^{-1/2} J D^{-1/2} \]
Note that the eigenvalues of $J_{\mathcal D}$ lie in the interval $[-1,1]$. 
\end{defn}
\begin{defn}
The \emph{$\delta$-sum-of squares threshold rank} of $J$ is defined to be $t_{\delta}(J_{D}):=\sum_{i:|\lambda_{i}|>\delta}\lambda_{i}^{2}$,
where $\lambda_{1},\dots,\lambda_{n}$ denote the eigenvalues of $J_{D}$. 
\end{defn}

Our methods will extend to the low sum-of-squares threshold
rank setting due to the following algorithmic regularity lemma of
Gharan and Trevisan.
\begin{theorem}\cite{gharan-trevisan}\label{reg-ghar-trev}
Let $J$ be the matrix of interaction strengths of an Ising
model, let
$\epsilon>0$ and let $t:=t_{\epsilon/2}(J_{D})$. There exists
a cut decomposition of $J$, $D = D^{(1)} + \cdots + D^{(s)}$, such that
$s \le 16 t/\epsilon^2$, 
\[ \|J - D\|_C \le \epsilon \|\vec{J}\|_1 \]
and $|d_i|\le \sqrt{t}/m$. Furthermore this decomposition can be computed
in $poly(n,t,1/\epsilon)$ time.
\end{theorem}
Now we can prove the main result of this section.
\begin{theorem}
Fix $\epsilon > 0$ and let $t = t_{\epsilon/2}(J_D)$ as in Theorem~\ref{reg-ghar-trev}, then
\[ \F - \mathcal{F}^* \le 3 \epsilon \|\vec{J}\|_1 + \frac{32t}{\epsilon^2} \log\left(\frac{2 \sqrt{t} n s}{\epsilon \|\vec J\|_1} + 1\right) \]
\end{theorem}
\begin{proof}
We mimic the proof of Theorem~\ref{thm-main-structural-result}.

We apply Theorem~\ref{reg-ghar-trev} to get a matrix $D = D^{(1)} + \cdots + D^{(s)}$ and let $\mathcal{F}_D$ and $\mathcal{F^*}_D$ denote the free energy and variational energy
of the Ising model with interaction matrix $D$; by \cref{lemma: free-energy-lipschitz} we
know
\[ |\mathcal{F}_D - \mathcal{F}| \le \epsilon \|\vec J\|_1, \|\mathcal{F}^*_D - \mathcal{F}| \le \epsilon\|\vec J\|_1. \]
Letting $Z_D$ denote the partition function of this Ising model, we see
\[
Z_{D}=\sum_{r,c}\exp\left(\sum_{i=1}^{s}r_{i}(x)c_{i}(x)d_{i}\right)\left(\sum_{x\in\{\pm1\}^{n}:r(x)=r,c(x)=c}1\right),
\]
where $r=(r_{1},\dots,r_{s})$ ranges over all elements of $[-|R_{1}|,|R_{1}|]\times\dots\times[-|R_{s}|,|R_{s}|]$
and similarly for $c$
Applying the argument from Lemma~\ref{lemma:gamma-def} now gives
\begin{lemma}\label{lemma:gamma-def2}
Let $J, D^{1},\dots,D^{s}$ be as above. Then, given real numbers $r_{i},r'_{i},c_{i},c'_{i}$ for each $i\in[s]$ and some 
$\upsilon \in (0,1)$ such that $r_i,c_i,r'_i,c'_i \le n$, 
$|r_i - r'_i| \le \upsilon n$ and $|c_i - c'_i| \le \upsilon n$ 
for all $i\in[s]$, we get that 
$\sum_i d_i|r'_i c'_i - r_i c_i| \le 2\sqrt{t}\upsilon ns$. 
\end{lemma}
As before we use this lemma to group terms.
Accordingly, for any $r \in [-|R_1|,|R_1|]\times\dots\times[-|R_s|,|R_s|]$, $c \in [-|C_1|,|C_1|]\times\dots\times[-|C_s|,|C_s|]$ and $\upsilon > 0$, let 
$$X_{r,c,\upsilon}:= \{x\in \{\pm1\}^{n}: |r_i(x)-r_i| \leq \upsilon n, |c_i(x)-c_i| \leq \upsilon n \text{ for all } i\in [s]\}.$$ 
Let $I_\upsilon := \{\pm\upsilon n, \pm 3\upsilon n,\pm 5\upsilon n,\dots,\pm\ell \upsilon n\}$, where $\ell$ is the smallest odd integer satisfying $|\ell \upsilon n - n| \leq \upsilon n$, so $|I_{\upsilon}| \le 1/\upsilon + 1$. Let 
$$Z_{D,\upsilon,\alpha}^{*}:=\max_{r,c\in I_{\upsilon}^{s}}\exp\left(\sum_{i=1}^{s}r_{i}c_{i}d_{i}+\log|X_{r,c,\alpha\upsilon}|\right).$$

Then by following the argument from \cref{thm-main-structural-result} we find
\begin{equation}
\label{eqn:ltr-approx-sum-by-max-lb}
\log Z_{D,\upsilon,1}^{*}\geq\log Z_{D}-2\sqrt{t}\upsilon ns - 2s\log |I_\upsilon| \ge \log Z_{D}-2\sqrt{t}\upsilon ns - 2s\log(1/\upsilon + 1)
\end{equation}
Finally, the argument from \cref{lemma:epsilon-bound} now gives
\[ \log Z^*_{D,\gamma,1} \le \mathcal{F}^*_D + 2\sqrt{t}\upsilon ns \]
and so letting $\upsilon = \frac{\epsilon \|\vec J\|_1}{2 \sqrt{t} n s}$ we find
\begin{align*}
\mathcal{F}^*_D 
&\ge \log Z_D - 2\sqrt{t}\upsilon ns - 2s\log(1/\upsilon + 1) \\
&\ge \log Z_D - \epsilon \|\vec{J}\|_1 - \frac{32t}{\epsilon^2} \log(\frac{2 \sqrt{t} n s}{\epsilon \|\vec J\|_1} + 1) \\
\end{align*}
and finally
\[ \mathcal{F} - \mathcal{F}^* \le 3 \epsilon \|\vec{J}\|_1 + \frac{32t}{\epsilon^2} \log\left(\frac{2 \sqrt{t} n s}{\epsilon \|\vec J\|_1} + 1\right) \]
\end{proof}

\end{document}